\RequirePackage{fix-cm}
\documentclass[smallextended]{svjour3}       
\smartqed  
\usepackage{graphicx}
\usepackage{mathtools,amssymb}
\usepackage{enumerate,adjustbox,color}

\usepackage[normalem]{ulem}
\newcommand{\mathbbm}[1]{\text{\usefont{U}{bbm}{m}{n}#1}}

\DeclareMathOperator\rk{rank}
\DeclareMathOperator\spn{span}

\DeclareMathOperator\C{C}
\DeclareMathOperator\join{Join}
\newcommand{\tp}{{\scriptscriptstyle\mathsf{T}}}

\begin{document}

\title{Best $k$-layer neural network approximations}


\author{Lek-Heng Lim \and
Mateusz Micha\l ek \and
Yang Qi}

\authorrunning{L.-H.~Lim, M.~Micha\l ek, Y.~Qi} 

\institute{L.-H.~Lim \at
Department of Statistics, University of Chicago, Chicago, IL 60637 \\
\email{lekheng@galton.uchicago.edu}
\and
M.~Micha\l ek \at
Max Planck Institute for Mathematics in the Sciences, 04103 Leipzig, Germany \\
Polish Academy of Sciences, 00656 Warsaw, Poland\\
\email{Mateusz.Michalek@mis.mpg.de}
\and
Y.~Qi \at
Department of Mathematics, University of Chicago, Chicago, IL 60637 \\
\email{yangqi@math.uchicago.edu}
}

\date{Received: date / Accepted: date}

\maketitle

\begin{abstract}
We show that the empirical risk minimization (ERM) problem for neural networks has no solution in general. Given a training set $s_1, \dots, s_n \in \mathbb{R}^p$ with corresponding responses $t_1,\dots,t_n \in \mathbb{R}^q$, fitting a $k$-layer neural network $\nu_\theta : \mathbb{R}^p \to \mathbb{R}^q$ involves estimation of the weights $\theta \in \mathbb{R}^m$ via an ERM:
\[
\inf_{\theta \in \mathbb{R}^m} \; \sum_{i=1}^n \lVert t_i - \nu_\theta(s_i) \rVert_2^2.
\]
We show that even for $k = 2$, this infimum is not attainable in general for common activations like ReLU, hyperbolic tangent, and sigmoid functions. A high-level explanation is like that for the  nonexistence of best rank-$r$ approximations of higher-order tensors --- the set of parameters is not a closed set --- but the geometry involved for best $k$-layer neural networks approximations is more subtle. In addition, we show that for smooth activations $\sigma(x)= 1/\bigl(1 + \exp(-x)\bigr)$ and $\sigma(x)=\tanh(x)$, such failure to attain an infimum can happen on  a positive-measured subset of responses. For the ReLU activation $\sigma(x)=\max(0,x)$, we completely classifying cases where the ERM for a best two-layer neural network approximation attains its infimum. As an aside, we obtain a precise description of the geometry of the space of two-layer neural networks with $d$ neurons in the hidden layer: it is the join locus of a line and the $d$-secant locus of a cone.
\end{abstract}

\section{Introduction}\label{sec:intro}

Let $\alpha_i : \mathbb{R}^{d_i} \to \mathbb{R}^{d_{i+1}}$, $x \mapsto A_i x + b_i$ be an affine function with $A_i \in \mathbb{R}^{d_{i+1} \times d_i}$ and $b_i \in \mathbb{R}^{d_{i+1}}$, $i =1,\dots,k$. Given any fixed \emph{activation} function $\sigma : \mathbb{R} \to \mathbb{R}$, we will abuse notation slightly by also writing $\sigma : \mathbb{R}^d \to \mathbb{R}^d$ for the function where $\sigma$ is applied coordinatewise, i.e., $\sigma(x_1,\dots,x_d) = (\sigma(x_1), \dots, \sigma(x_d))$, for any $d \in \mathbb{N}$. Consider a $k$-layer neural network $\nu : \mathbb{R}^{p} \to \mathbb{R}^{q}$,
\begin{equation}\label{eq:knn}
\nu =\alpha_k \circ \sigma \circ \alpha_{k-1} \circ \dots \circ \sigma \circ \alpha_2 \circ \sigma \circ \alpha_1,
\end{equation}
obtained from alternately composing $\sigma$ with affine functions  $k$ times. Note that such a function $\nu$ is parameterized (and completely determined) by its \emph{weights} $\theta \coloneqq (A_k,b_k, \dots, A_1, b_1)$ in
\begin{equation}\label{eq:weights}
\Theta \coloneqq ( \mathbb{R}^{d_{k+1} \times d_k} \times \mathbb{R}^{d_{k+1}} ) \times \dots \times ( \mathbb{R}^{d_2 \times d_1} \times \mathbb{R}^{d_2} ) \cong \mathbb{R}^{m}.
\end{equation}
Here and throughout this article,
\begin{equation}\label{eq:dim}
m \coloneqq \sum_{i=1}^k  (d_i + 1)d_{i+1}
\end{equation}
will always denote the number of weights that parameterize $\nu$;
\[
p \coloneqq d_1 \qquad \text{and} \qquad  q \coloneqq d_{k+1}
\]
will always denote the dimensions of the first and last layers. In neural networks lingo, the dimension of the $i$th layer $d_i$ is also called the \emph{number of neurons} in the $i$th layer.
Whenever it is necessary to emphasize the dependence of $\nu$ on $\theta$, we will write $\nu_\theta$ for a $k$-layer neural network parameterized by $\theta \in \Theta$.

Traditional studies of neural networks in approximation theory typically focus on the \emph{function approximation problem} with\footnote{Results may be extended to $q > 1$ by applying them coordinatewise.} $q=1$, i.e., given a target function $f : \mathbb{R}^p \to \mathbb{R}$ in some Banach space $\mathcal{B}$, how well can it be approximated by a neural network $\nu_\theta : \mathbb{R}^p \to \mathbb{R}$ in the Banach space norm $\lVert\, \cdot\, \rVert_{\mathcal{B}}$? In other words,  one is interested in the problem
\begin{equation}\label{eq:approxBanach}
\inf_{\theta \in \Theta} \; \lVert f - \nu_\theta \rVert_{\mathcal{B}}.
\end{equation}
The most celebrated results along these lines are the universal approximation theorems of Cybenko \cite{Cybenko}, for sigmoidal activation and $L^1$-norm, as well as those of Hornik et al.\ \cite{Hornik89,Hornik91}, for more general activations such as ReLU and $L^p$-norms, $1 \le p \le \infty$. These results essentially say that the infimum in \eqref{eq:approxBanach} is zero as long as $k$ is at least two (but with no bound on $d_2$).

Nevertheless, in reality, one does not solve the problem \eqref{eq:approxBanach}  when training a neural network but a parameter estimation problem called the \emph{empirical risk minimization problem}. Let $s_1,\dots,s_n \in \mathbb{R}^p$ be a sample of $n$ independent, identically distributed observations with corresponding \emph{responses} $t_1,\dots,t_n \in \mathbb{R}^q$. The main computational problem in supervised learning with neural networks is to fit the \emph{training set} $\{(s_i, t_i) \in \mathbb{R}^p \times \mathbb{R}^q : i = 1,\dots,n\}$ with a $k$-layer neural network $\nu_\theta :\mathbb{R}^p \to \mathbb{R}^q$ so that
\[
t_i \approx \nu_\theta(s_i), \quad i =1,\dots,n,
\]
often in the least-squares sense
\begin{equation}\label{eq:approx1}
\inf_{\theta \in \Theta}\;  \sum_{i=1}^n \lVert t_i - \nu_\theta(s_i) \rVert_2^2.
\end{equation}
The responses are regarded as values of the unknown function $f : \mathbb{R}^p \to \mathbb{R}^q$ to be learned, i.e., $t_i = f(s_i)$, $i =1,\dots,n$. The hope is that by solving \eqref{eq:approx1} for $\theta^* \in \mathbb{R}^m$, the neural network obtained $\nu_{\theta^*}$ will approximate $f$ well in the sense of having small generalization errors, i.e., $f(s) \approx \nu_{\theta^*}(s)$ for $s \notin \{s_1,\dots,s_n\}$. This hope has been borne out empirically in spectacular ways \cite{Alex,LBH,Go}.

The problem \eqref{eq:approxBanach} asks how well a given target function can be approximated by a given function class, in this case the class of $k$-layer $\sigma$-activated neural networks. This is an infinite-dimensional problem. On the other hand \eqref{eq:approx1} simply asks how well the approximation is at finitely many sample points, a finite-dimensional problem, and therefore amenable to techniques in geometry.

In \cite{QML}, we applied methods from algebraic and differential geometry to study the empirical risk minimization problem corresponding to \emph{nonlinear approximation}, i.e., where one seeks to approximate a target function by a sum of $k$ atoms $\varphi_1,\dots,\varphi_k$ from a dictionary $D$,
\[
\inf_{\varphi_i \in D} \; \lVert f - \varphi_1 - \varphi_2 - \dots - \varphi_k \rVert_{\mathcal{B}}.
\]
If we denote the layers of a neural network by $\varphi_i \in L$, then \eqref{eq:approxBanach} may be written in a form that parallels the above:
\[
\inf_{\varphi_i \in L} \; \lVert f - \varphi_k \circ \varphi_{k-1} \circ \dots \circ \varphi_1 \rVert_{\mathcal{B}}.
\]
Again our goal is to study the corresponding empirical risk minimization problem, i.e., the approximation problem \eqref{eq:approx1}. The first surprise is that this may not always have a solution. For example, take $n = 6$ and $p=q=2$ with
\begin{alignat*}{12}
s_1 &= \begin{bmatrix*}[r] -2 \\ 0 \end{bmatrix*},\;  & 
s_2 &= \begin{bmatrix*}[r] -1 \\ 0 \end{bmatrix*},\;  & 
s_3 &= \begin{bmatrix*}[r] 0 \\ 0 \end{bmatrix*},\;  & 
s_4 &= \begin{bmatrix*}[r] 1 \\ 0 \end{bmatrix*},\;  & 
s_5 &= \begin{bmatrix*}[r] 2 \\ 0 \end{bmatrix*},\;  & 
s_6 &= \begin{bmatrix*}[r] 1 \\ 1 \end{bmatrix*},\\
t_1 &= \begin{bmatrix*}[r] 2 \\ 0 \end{bmatrix*},\;  & 
t_2 &= \begin{bmatrix*}[r] 1 \\ 0 \end{bmatrix*},\;  & 
t_3 &= \begin{bmatrix*}[r] 0 \\ 0 \end{bmatrix*},\;  & 
t_4 &= \begin{bmatrix*}[r] -2 \\ 0 \end{bmatrix*},\;  & 
t_5 &= \begin{bmatrix*}[r] -4 \\ 0 \end{bmatrix*},\;  & 
t_6 &= \begin{bmatrix*}[r] 0 \\ 1 \end{bmatrix*}.
\end{alignat*}
For a ReLU-activated two-layer neural network, the approximation problem \eqref{eq:approx1} seeks weights $\theta = (A,b,B,c)$  that attain the infimum over all  $A, B \in \mathbb{R}^{2 \times 2}$, $b,c \in \mathbb{R}^2$ of the loss function
\[
\begin{adjustbox}{width=\textwidth,totalheight=\textheight,keepaspectratio}
$\begin{aligned}
&\biggl\lVert
\begin{bmatrix*}[r]
2\\
0
\end{bmatrix*} -  \biggl[
 B \max\biggl(A
\begin{bmatrix*}[r] -2
\\ 0
\end{bmatrix*}  + b, \begin{bmatrix*}[r] 0 \\ 0 \end{bmatrix*} \biggr) + c \biggr] \biggr\rVert^2
+
\biggl\lVert
\begin{bmatrix*}[r]
1\\
0
\end{bmatrix*} -  \biggl[
 B \max\biggl(A
\begin{bmatrix*}[r] -1
\\ 0
\end{bmatrix*}  + b, \begin{bmatrix*}[r] 0 \\ 0 \end{bmatrix*} \biggr) + c \biggr] \biggr\rVert^2
\\
&\quad+\biggl\lVert
\begin{bmatrix*}[r]
0\\
0
\end{bmatrix*} -  \biggl[
 B \max\biggl(A
\begin{bmatrix*}[r] 0
\\ 0
\end{bmatrix*}  + b, \begin{bmatrix*}[r] 0 \\ 0 \end{bmatrix*} \biggr) + c \biggr] \biggr\rVert^2
+
\biggl\lVert
\begin{bmatrix*}[r]
-2\\
0
\end{bmatrix*} -  \biggl[
 B \max\biggl(A
\begin{bmatrix*}[r] 1
\\ 0
\end{bmatrix*}  + b, \begin{bmatrix*}[r] 0 \\ 0 \end{bmatrix*} \biggr) + c \biggr] \biggr\rVert^2
\\
&\quad\quad+\biggl\lVert
\begin{bmatrix*}[r]
-4\\
0
\end{bmatrix*} -  \biggl[
 B \max\biggl(A
\begin{bmatrix*}[r] 2
\\ 0
\end{bmatrix*}  + b, \begin{bmatrix*}[r] 0 \\ 0 \end{bmatrix*} \biggr) + c \biggr] \biggr\rVert^2
+
\biggl\lVert
\begin{bmatrix*}[r]
0\\
1
\end{bmatrix*} -  \biggl[
 B \max\biggl(A
\begin{bmatrix*}[r] 1
\\ 1
\end{bmatrix*}  + b, \begin{bmatrix*}[r] 0 \\ 0 \end{bmatrix*} \biggr) + c \biggr] \biggr\rVert^2.
\end{aligned}$
\end{adjustbox}
\]
We  will see in the proof of Theorem~\ref{thm:nonclosed} that  this has no solution. Any sequence of $\theta = (A,b,B,c)$ chosen so that the loss function  converges to its infimum will have $\lVert\theta \rVert^2 = \lVert A \rVert_F^2 + \lVert b \rVert_2^2 + \lVert B\rVert_F^2 + \lVert c\rVert_2^2$ becoming unbounded --- the entries of $\theta$ will diverge to $\pm \infty$ in such a way that keeps the loss function bounded and in fact convergent to its infimum.

Note that we have assumed the Euclidean norm
\[
\lVert \theta \rVert^2 \coloneqq \sum_{i=1}^k \lVert A_i \rVert^2_F + \lVert b_i \rVert_2^2
\]
on $\Theta$ but the results in this article will be independent of the choice of norms (as all norms are equivalent on finite-dimensional spaces).

\section{Geometry of  empirical risk minimization for neural networks}

Given that we are interested in the behavior of $\nu_{\theta}$  as a function of weights $\theta$, we will rephrase \eqref{eq:approx1} to put it on more relevant footing. Let the sample $s_1,\dots,s_n \in \mathbb{R}^p$ and  responses $t_1,\dots,t_n \in \mathbb{R}^q$  be arbitrary but fixed. Henceforth we will assemble the sample into a \emph{design matrix},
\begin{equation}\label{eq:design}
S \coloneqq \begin{bmatrix} s_1^\tp \\ s_2^\tp \\ \vdots \\ s_n^\tp \end{bmatrix} \in \mathbb{R}^{n\times p},
\end{equation}
and the corresponding responses into a \emph{response matrix},
\begin{equation}\label{eq:target}
T \coloneqq \begin{bmatrix} t_1^\tp \\ t_2^\tp \\ \vdots \\ t_n^\tp \end{bmatrix} \in \mathbb{R}^{n\times q}.
\end{equation}
Here and for the rest of this article, we use the following numerical linear algebra conventions:
\begin{itemize}
\item a vector $a \in \mathbb{R}^n$ will always be regarded as a column vector; 

\item a row vector will always be denoted $a^\tp$ for some column vector $a$;

\item a matrix $A \in \mathbb{R}^{n \times p}$ is denoted as a list of its column vectors in the form $A = [a_1,\dots,a_p]$;

\item or as a list of its row vectors in the form  $A = [\alpha_1^\tp,\dots,\alpha_n^\tp]^\tp$.
\end{itemize}
We will also adopt the convention that treats a \emph{direct sum} of $p$ subspaces (resp.\ cones) in $\mathbb{R}^n$ as a subspace (resp.\ cone) of $\mathbb{R}^{n \times p}$: If $V_1,\dots,V_p \subseteq \mathbb{R}^n$ are subspaces (resp.\ cones), then
\begin{equation}\label{eq:ds}
V_1 \oplus \dots \oplus V_p \coloneqq \{ [ v_1,\dots, v_p ] \in \mathbb{R}^{n \times p} : v_1 \in V_1, \dots, v_p \in V_p\}.
\end{equation}

Let $\nu_\theta : \mathbb{R}^p \to \mathbb{R}^q$ be a $k$-layer neural network, $k \ge 2$. We define the \emph{weights map} $\psi_k : \Theta  \to \mathbb{R}^{n \times q}$ by
\[
\psi_k(\theta) = \begin{bmatrix} \nu_\theta(s_1)^\tp\\ \nu_\theta(s_2)^\tp \\ \vdots \\ \nu_\theta(s_n)^\tp \end{bmatrix}  \in \mathbb{R}^{n\times q}.
\]
In other words, for a fixed sample, $\psi_k$ is $\nu_\theta$ regarded as a function of the weights $\theta$.
The empirical risk minimization problem  is \eqref{eq:approx1} rewritten as
\begin{equation}\label{eq:approx2}
\inf_{\theta \in \Theta}\; \lVert T - \psi_k(\theta) \rVert_F,
\end{equation}
where $\lVert \, \cdot \, \rVert_F$ denotes the Frobenius norm. We may view \eqref{eq:approx2} as a matrix approximation problem --- finding a matrix in
\begin{equation}\label{eq:image}
\psi_k(\Theta) = \{\psi_k(\theta) \in \mathbb{R}^{n\times q} : \theta \in \Theta\}
\end{equation}
that is nearest to a given matrix $T \in \mathbb{R}^{n \times q}$.
\begin{definition}\label{def:image}
We will call the set $\psi_k(\Theta)$ the \emph{image of weights} of the $k$-layer neural network $\nu_\theta$ and the corresponding problem \eqref{eq:approx2} a \emph{best $k$-layer neural network approximation problem}.
\end{definition}
As we noted in \eqref{eq:weights}, the space of all weights $\Theta$ is essentially the Euclidean space $\mathbb{R}^m$ and uninteresting geometrically, but the image of weights $\psi_k(\Theta)$, as we will see in this article, has complicated geometry (e.g., for $k=2$ and ReLU activation, it is the join locus of a line and the secant locus of a cone --- see Theorem~\ref{thm:ab1}). In fact, the geometry of the neural network \emph{is} the geometry of  $\psi_k(\Theta)$. We expect that it will be pertinent to understand this geometry if one wants to understand neural networks at a deeper level. For one, the nontrivial geometry of $\psi_k(\Theta)$ is the reason that the best $k$-layer neural network approximation problem, which is to find a point in  $\psi_k(\Theta)$ closest to a given $T \in \mathbb{R}^{n \times q}$, lacks a solution in general.

Indeed, the most immediate mathematical  issues with the approximation problem \eqref{eq:approx2} are the existence and uniqueness of solutions:
\begin{enumerate}[\upshape (i)]
\item\label{exist} a nearest point may not exist since the set $\psi_k(\Theta) $ may not be a closed subset of $\mathbb{R}^{n \times q}$, i.e., the infimum in \eqref{eq:approx2} may not be attainable;

\item\label{unique} even if it exists, the nearest point may not be unique, i.e., the infimum in \eqref{eq:approx2} may be  attained by two or more points in $\psi_k(\Theta)$.
\end{enumerate}
As a reminder a problem is said to be \emph{ill-posed} if it lacks existence and uniqueness guarantees. Ill-posedness creates numerous difficulties both logical (what does it mean to find a solution when it does not exist?) and practical (which solution do we find when there are more than one?). In addition, a well-posed problem near an ill-posed one is the very definition of an \emph{ill-conditioned} problem \cite{condition}, which presents its own set of difficulties. In general, ill-posed problems are not only to be avoided but also delineated to reveal the region of ill-conditioning.

For the function approximation problem \eqref{eq:approxBanach}, the nonexistence issue of a best neural network approximant is very well-known, dating back to \cite{GP}. But for the best $k$-layer neural network approximation problem \eqref{eq:approx1} or \eqref{eq:approx2}, its well-posedness has never been studied, to the best of our knowledge. Our article seeks to address this gap.  The geometry of the set in \eqref{eq:image} will play an important role in studying these problems, much like the role played by the geometry of rank-$k$ tensors in \cite{QML}.

We will  show that for many networks, the problem \eqref{eq:approx2} is ill-posed. We have already mentioned an explicit example at the end of Section~\ref{sec:intro} for the ReLU activation:
\begin{equation}\label{eq:relu}
\sigma_{\max}(x)\coloneqq \max(0,x)
\end{equation}
where \eqref{eq:approx2} lacks a solution; we will discuss this in detail in Section~\ref{sec:shallow}. Perhaps more surprisingly, for the  sigmoidal and hyperbolic tangent activation functions:
\[
\sigma_{\exp}(x)\coloneqq \frac{1}{1 + \exp(-x)}\quad\text{and}\quad\sigma_{\tanh}(x)\coloneqq \tanh(x),
\]
we will see in Section~\ref{sec:smooth} that \eqref{eq:approx2} lacks a solution with positive probability, i.e., there exists an open set $U\subseteq \mathbb{R}^{n \times q}$ such that for any $T\in U$ there is no nearest point in $\psi_k(\Theta)$. Similar phenomenon is known for real tensors \cite[Section~8]{dSL}. 

For neural networks with ReLU activation, we are unable to establish similar ``failure with positive probability'' results but the geometry of the problem \eqref{eq:approx2} is actually simpler in this case. For two-layer ReLU-activated network, we can  completely characterize the geometry of $\psi_2(\Theta)$, which provides us with greater insights as to why \eqref{eq:approx2} generally lacks a solution. We can also determine the  dimension of  $\psi_2(\Theta)$ in many instances. These will be discussed in Section~\ref{sec:geom}.

The following map will play a key role in this article and we give it a name to facilitate exposition.
\begin{definition}[ReLU projection]
The map $\sigma_{\max} : \mathbb{R}^d \to \mathbb{R}^d$  where the ReLU activation \eqref{eq:relu} is applied coordinatewise will be called a \emph{ReLU projection}.
For any $\Omega \subseteq \mathbb{R}^d$, $\sigma_{\max}(\Omega) \subseteq \mathbb{R}^d$ will be called a ReLU projection of $\Omega$. 
\end{definition}
Note that a ReLU projection is a linear projection when restricted to any orthant of $\mathbb{R}^d$. 

\section{Geometry of a ``one-layer neural network''}\label{sec:one}

We start by studying the `first part' of a two-layer ReLU-activated neural network:
\[
\mathbb{R}^p \xrightarrow{\alpha}\mathbb{R}^q \xrightarrow{\sigma_{\max}}\mathbb{R}^q
\]
and, slightly abusing terminologies, call this a \emph{one-layer ReLU-activated neural network}.
Note that the weights here are $\theta = (A,b) \in \mathbb{R}^{q \times (p+1)} = \Theta$ with $A \in \mathbb{R}^{q \times p}$, $b \in \mathbb{R}^q$ that define the affine map $\alpha(x) = Ax+ b$.

Let the sample  $S = [s_1^\tp,\dots,s_n^\tp]^\tp \in \mathbb{R}^{n \times p}$ be fixed. Define the weight map $\psi_1: \Theta \to \mathbb{R}^{n \times q}$ by
\begin{equation}\label{eq:psi1}
\psi_1(A,b) = \sigma_{\max} \left(\begin{bmatrix} A s_1 + b \\  \vdots \\ A s_n +b \end{bmatrix} \right),
\end{equation}
where $\sigma_{\max}$ is applied coordinatewise.

Recall that a  cone $C \subseteq \mathbb{R}^d$ is simply a set invariant under scaling by positive scalars, i.e., if $x \in C$, then $\lambda x \in C$ for all $\lambda > 0$. The dimension of a cone is the dimension of the smallest subspace that contains it, i.e., $\dim C = \dim \spn(C)$.
\begin{definition}[ReLU cone]\label{def:ReLUcone}
Let $S = [s_1^\tp,\dots,s_n^\tp]^\tp \in \mathbb{R}^{n \times p}$. The \emph{ReLU cone} of $S$ is the set
\[
\C_{\max}(S) \coloneqq \left\{ \sigma_{\max} \left(\begin{bmatrix} a^\tp  s_1 + b \\  \vdots \\ a^\tp s_n +b \end{bmatrix} \right) \in \mathbb{R}^n : a \in \mathbb{R}^p, \; b\in \mathbb{R} \right\}.
\]
\end{definition}
The ReLU cone is clearly a cone. Such cones will form the building blocks for the  image of weights $\psi_k(\Theta)$. In fact, it is easy to see that $\C_{\max}(S)$ is exactly $\psi_1(\Theta)$ in case when $q =1$. The next lemma describes the geometry of ReLU cones in greater detail.

\begin{lemma}\label{lem:veryshort}
Let $S = [s_1^\tp,\dots,s_n^\tp]^\tp \in \mathbb{R}^{n \times p}$.
\begin{enumerate}[\upshape (i)]
\item The set $\C_{\max}(S)$ is always a closed pointed cone of dimension
\begin{equation}\label{eq:dim}
\dim \C_{\max}(S)  = \rk [S, \mathbbm{1}].
\end{equation}
Here $[S, \mathbbm{1}] \in \mathbb{R}^{n \times (p+1)}$ is augmented with an extra column $\mathbbm{1}  \coloneqq [1,\dots,1]^\tp \in \mathbb{R}^n$, the vector of all ones.

\item A set $C \subseteq \mathbb{R}^n$ is a ReLU cone if and only if it is a ReLU projection of some linear subspace in $\mathbb{R}^n$ containing the vector $\mathbbm{1}$.
\end{enumerate}
\end{lemma}
\begin{proof}
Consider the map in \eqref{eq:psi1} with $q=1$. Then $\psi_1 :\Theta \to \mathbb{R}^n$ is given by a composition of the linear map
\[
\Theta \to \mathbb{R}^n, \quad \begin{bmatrix}a \\ b \end{bmatrix} \mapsto \begin{bmatrix} s_1^\tp a + b \\ \vdots \\ s_n^\tp a + b \end{bmatrix} =  [S, \mathbbm{1}] \begin{bmatrix} a\\ b\end{bmatrix} ,
\]
whose image  is a linear space $L$ of dimension $\rk [S, \mathbbm{1}]$ containing $\mathbbm{1}$, and  the  ReLU projection $\sigma_{\max} : \mathbb{R}^n\rightarrow\mathbb{R}^n$. Since $\C_{\max}(S) = \psi_1(\Theta) $, it is a ReLU projection of a linear subspace in $\mathbb{R}^n$, which is clearly a closed pointed cone.

For its dimension, note that a ReLU projection is a linear projection on each quadrant and thus cannot increase dimension. On the other hand, since $\mathbbm{1}\in L$, we know that $L$ intersects the interior of the nonnegative quadrant on which $\sigma_{\max}$ is the identity; thus $\sigma_{\max}$ preserves dimension and we have \eqref{eq:dim}. 

Conversely, a ReLU projection of any linear space $L$ may be realized as $\C_{\max}(S) $ for some choice of $S$ --- just choose $S$ so that the image of the matrix $[S, \mathbbm{1}]$ is $L$. \qed
\end{proof}

It follows from Lemma~\ref{lem:veryshort} that the weights of a one-layer ReLU neural network has the geometry of a direct sum of $q$ closed pointed cones. Recall our convention for direct sum in \eqref{eq:ds}.
\begin{corollary}\label{cor:product}
Consider the one-layer ReLU-activated neural network
\[
\mathbb{R}^p \xrightarrow{\alpha_1}\mathbb{R}^q \xrightarrow{\sigma_{\max}}\mathbb{R}^q.
\]
Let $S \in \mathbb{R}^{n \times p}$. Then $\psi_1(\Theta)\subseteq \mathbb{R}^{n \times q}$ has the structure of a direct sum of $q$ copies of $\C_{\max}(S) \subseteq \mathbb{R}^n$. More precisely,
\begin{equation}\label{eq:cmaxq}
\psi_1(\Theta) = \{ [v_1,\dots, v_q] \in \mathbb{R}^{n \times q} : v_1,\dots,v_q \in  \C_{\max}(S) \}.
\end{equation}
In particular, $\psi_1(\Theta)$ is a closed pointed cone of dimension $q \cdot \rk [S, \mathbbm{1}]$ in $\mathbb{R}^{n \times q}$.
\end{corollary}
\begin{proof}
Each row of the matrix $\alpha_1$ can be identified with the affine map defined in Lemma~\ref{lem:veryshort}. Then the conclusion follows by Lemma~\ref{lem:veryshort}. \qed
\end{proof}
Given Corollary~\ref{cor:product}, one might perhaps think that a two-layer neural network
\[
\mathbb{R}^{d_1}\xrightarrow{\alpha_1}\mathbb{R}^{d_2}\xrightarrow{\sigma_{\max}}\mathbb{R}^{d_2}\xrightarrow{\alpha_2}\mathbb{R}^{d_3}
\]
would also have a closed image of weights $\psi_2(\Theta)$. This turns out to be false. We will show that the image $\psi_2(\Theta)$ may not  be closed. 

As a side remark, note that Definition~\ref{def:ReLUcone} and Lemma~\ref{lem:veryshort} are peculiar to the ReLU activation. For smooth activations like $\sigma_{\exp}$ and $\sigma_{\tanh}$, the image of weights  $\psi_k(\Theta)$ is almost never a cone and Lemma~\ref{lem:veryshort} does not hold in multiple ways.

\section{Ill-posedness of best $k$-layer neural network approximation}\label{sec:shallow}

The $k = 2$ case is the simplest and yet already nontrivial in that it has the universal approximation property, as we mentioned earlier.

\begin{theorem}[Ill-posedness of neural network approximation I]\label{thm:nonclosed}
The best two-layer neural network approximation problem
\begin{equation}\label{eq:approx3}
\inf_{\theta \in \Theta}\; \lVert T - \psi_2(\theta) \rVert_F,
\end{equation}
is ill-posed, i.e., the infimum in \eqref{eq:approx3} cannot be attained in general.
\end{theorem}

\begin{proof}
We will construct an explicit two-layer ReLU-activated network whose image of weights is not closed.
Let $d_1 = d_2 = d_3 = 2$. For the  two-layer ReLU-activated network
\[
\mathbb{R}^2\xrightarrow{\alpha_1}\mathbb{R}^2\xrightarrow{\sigma_{\max}}\mathbb{R}^2\xrightarrow{\alpha_2}\mathbb{R}^2,
\]
the weights take the form
\[
\theta = (A_1, b_1, A_2, b_2) \in \mathbb{R}^{2 \times 2} \times \mathbb{R}^2 \times \mathbb{R}^{2 \times 2} \times \mathbb{R}^2  = \Theta \cong \mathbb{R}^{12},
\]
i.e., $m = 12$. Consider a sample $S $ of size $n = 6$  given by
\[
s_i=\begin{bmatrix}i-3 \\ 0\end{bmatrix}, \quad i=1,\dots,5,\quad s_6= \begin{bmatrix}1 \\1\end{bmatrix},
\]
that is
\[
S=\begin{bmatrix*}[r] -2 & 0 \\ -1 & 0 \\ 0 & 0 \\ 1 & 0 \\ 2 & 0 \\ 1 & 1\end{bmatrix*}  \in \mathbb{R}^{6 \times 2}.
\]
Thus the weight map $\psi_2 : \Theta \to \mathbb{R}^{6 \times 2}$, or, more precisely,
\[
\psi_2 :   \mathbb{R}^{2 \times 2} \times \mathbb{R}^2 \times \mathbb{R}^{2 \times 2} \times \mathbb{R}^2  \to \mathbb{R}^{6 \times 2},
\]
is given by
\[
\psi_2(\theta) = \begin{bmatrix}\nu_\theta(s_1)^\tp \\ \vdots \\ \nu_\theta(s_6)^\tp\end{bmatrix}
= \begin{bmatrix}\bigl( A_2 \max(A_1 s_1 + b_1, 0 ) + b_2\bigr)^\tp \\ \vdots \\ \bigl(A_2 \max(A_1 s_6 + b_1, 0 ) + b_2\bigr)^\tp\end{bmatrix} \in \mathbb{R}^{6 \times 2}.
\]
We claim that the image of weights $\psi_2(\Theta)$ is not closed --- the point
\begin{equation}\label{eq:t}
T = \begin{bmatrix*}[r] 2 & 0 \\ 1 & 0 \\ 0 & 0 \\ -2 & 0 \\ -4 & 0 \\ 0 & 1\end{bmatrix*}  \in \mathbb{R}^{6 \times 2}
\end{equation}
is in the closure of $\psi_2(\Theta)$  but not in $\psi_2(\Theta)$. Therefore for this choice of $T$, the infimum in \eqref{eq:approx3} is zero but is never attainable by any point in $\psi_2(\Theta)$.

We will first prove that $T$ is in the closure of $\psi_2(\Theta)$: Consider a sequence of affine transformations $\alpha_1^{(k)}$, $k =1,2,\dots,$ defined by
\[
\alpha_1^{(k)}(s_i)=(i-3)\begin{bmatrix}-1 \\ 1 \end{bmatrix}, \quad i=1,\dots,5, \quad \alpha_1^{(k)}(s_6)=\begin{bmatrix} 2k \\ k \end{bmatrix}, 
\]
and set
\[
\alpha_2^{(k)}\biggl(\begin{bmatrix} x\\ y \end{bmatrix}\biggr)=\begin{bmatrix} x-2y \\ \frac{1}{k}y \end{bmatrix}.
\]
The sequence of two-layer neural networks,
\[
\nu^{(k)} = \alpha_2^{(k)} \circ \sigma_{\max} \circ \alpha_1^{(k)}, \quad k \in \mathbb{N},
\]
have weights given by
\begin{equation}\label{eq:weightsdiverge}
\theta_k = \biggl( \begin{bmatrix} -1 & 2k + 1 \\ 1 & k - 1 \end{bmatrix}, \begin{bmatrix} 0 \\ 0 \end{bmatrix}, \begin{bmatrix} 1 & -2 \\ 0 & \frac{1}{k} \end{bmatrix}, \begin{bmatrix}0 \\ 0\end{bmatrix} \biggr) 
\end{equation}
and that
\[
\lim_{k\to \infty} \psi_2( \theta_k) = T.
\]
This shows that $T$ in \eqref{eq:t} is indeed in the closure of $\psi_2(\Theta)$.

We next show by contradiction that $T \notin \psi_2(\Theta)$. Write $T = [t_1^\tp,\dots,t_6^\tp]^\tp \in \mathbb{R}^{6 \times 2}$ where $t_1,\dots,t_6 \in \mathbb{R}^2$ are as in \eqref{eq:t}. Suppose $T \in \psi_2(\Theta)$. Then there exist some affine maps $\beta_1,\beta_2:\mathbb{R}^2\to\mathbb{R}^2$ such that
\[
t_i =  \beta_2 \circ \sigma_{\max} \circ \beta_1(s_i), \quad i =1,\dots,6.
\]
As $t_1, t_3, t_6$ are affinely independent, $\beta_2$ has to be an affine isomorphism. Hence the five points
\begin{equation}\label{eq:five2}
\sigma_{\max}(\beta_1(s_1)),\dots,\sigma_{\max}(\beta_1(s_5))
\end{equation}
have to lie on a line in $\mathbb{R}^2$. Also, note that
\begin{equation}\label{eq:five1}
\beta_1(s_1),\dots,\beta_1(s_5)
\end{equation}
lie on a (different) line in $\mathbb{R}^2$ since $\beta_1$ is an affine homomorphism. The five points in \eqref{eq:five1} have to be in the same quadrant, otherwise the points in \eqref{eq:five2} could not lie on a line or have successive distances $\delta,\delta,2\delta,2\delta$  for some $\delta > 0$. Note that $\sigma_{\max}:\mathbb{R}^2 \to \mathbb{R}^2$ is identity in the first quadrant, projection to the $y$-axis in the second, projection to the point $(0,0)$ in the third, and projection to the $x$-axis in the fourth. So $\sigma_{\max}$ takes points with equal successive distances in the first, second,  fourth quadrant to points with equal successive distances in the same quadrant; and  it takes all points in the third quadrant to the origin. Hence $\sigma_{\max}$ cannot take the five colinear points in \eqref{eq:five1}, with equal successive distances, to the five colinear points in \eqref{eq:five2}, with successive distances $\delta,\delta,2\delta,2\delta$. This yields the required contradiction.  \qed
\end{proof}

We would like to emphasize that the example constructed in the proof of Theorem~\ref{thm:nonclosed},  which is about the nonclosedness of the image of weights within a finite-dimensional space of response matrices, differs from examples in \cite{GP}, which are about the nonclosedness of the class of neural network within an infinite-dimensional space of target functions. Another difference is that here we have considered the ReLU activation $\sigma_{\max}$  as opposed to the hyperbolic tangent activation $\sigma_{\exp}$ in \cite{GP}.

As we pointed out at the end of Section~\ref{sec:intro} and as the reader might also have observed in the proof of Theorem~\ref{thm:nonclosed}, the sequence of weights $\theta_k$ in \eqref{eq:weightsdiverge} contain entries that become unbounded as $k \to \infty$. This is not peculiar to the sequence we chose in \eqref{eq:weightsdiverge}; by the same discussion in \cite[Section~4.3]{dSL}, this will always be the case:
\begin{proposition}
If the infimum in \eqref{eq:approx3} is not attainable, then any sequence of weights $\theta_k \in \Theta$ with
\[
\lim_{k\to\infty}  \lVert T - \psi_2(\theta_k) \rVert_F = \inf_{\theta \in \Theta} \; \lVert T - \psi_2(\theta) \rVert_F
\]
must be unbounded, i.e.,
\[
\limsup_{k\to\infty}  \; \lVert \theta_k \rVert = \infty.
\]
\end{proposition}

\section{Geometry of two-layer neural networks}\label{sec:geom}

There is a deeper geometrical explanation behind Theorem~\ref{thm:nonclosed}. Let $d \in \mathbb{N}$. The \emph{join locus} of $X_1,\dots,X_r \subseteq \mathbb{R}^d$ is the set
\begin{equation}\label{eq:joinvar}
\join(X_1, \dots, X_r) \coloneqq \{ x_1 + \dots + x_r \in  \mathbb{R}^d : x_i \in X_i, \; i=1,\dots, r\}.
\end{equation}
A special case is when $X_1 = \dots = X_r =X$ and in which case the join locus is called the $r$th \emph{secant locus}
\[
\Sigma^{\circ}_r(X)  = \{x_1 + \dots + x_r \in  \mathbb{R}^d  : x_1,\dots,x_r \in X\}.
\]
An example of a join locus is the set of ``sparse-plus-low-rank'' matrices \cite[Section~8.1]{QML}; an example of a $r$th secant locus is the set of rank-$r$ tensors \cite[Section~7]{QML}.

From a geometrical perspective, we will next show that for $k=2$ and $q=1$ the set $\psi_2(\Theta)$  has the structure of a join locus. Join loci are known in general to be nonclosed \cite{Zak}. For this reason, the ill-posedness of the best $k$-layer neural network problem is not unlike that of the best rank-$r$ approximation problem for tensors, which is  a consequence of the nonclosedness of the secant loci of the Segre variety \cite{dSL}.

We shall begin with the case of a two-layer neural network with one-dimensional output, i.e., $q =1$. In this case,  $\psi_2(\Theta) \subseteq \mathbb{R}^n$ and  we can describe its geometry very precisely. The more general case where $q > 1$ will be in Theorem~\ref{thm:abc}.
\begin{theorem}[Geometry of two-layer neural network I]\label{thm:ab1}
Consider the two-layer network with $p$-dimensional inputs and $d$ neurons in the hidden layer:
\begin{equation}\label{eq:2layernet}
\mathbb{R}^{p} \xrightarrow{\alpha_1}\mathbb{R}^{d}\xrightarrow{\sigma_{\max}}\mathbb{R}^{d}\xrightarrow{\alpha_2} \mathbb{R}.
\end{equation}
The image of weights is given by
\[
\psi_2(\Theta) = \join\bigl(\Sigma_{d}^{\circ}(\C_{\max}(S)) , \spn \{ \mathbbm{1} \}\bigr)
\]
where
\[
\Sigma_{d}^{\circ}\bigl(\C_{\max}(S)\bigr) \coloneqq \bigcup_{y_1,\dots,y_{d}\in \C_{\max}(S)} \spn\{ y_1,\dots, y_{d}\} \subseteq \mathbb{R}^n
\]
is the $d$th secant locus of the ReLU cone and $\spn \{ \mathbbm{1} \}$ is the one-dimensional subspace spanned by $\mathbbm{1}\in\mathbb{R}^n$.
\end{theorem}
\begin{proof}
Let $\alpha_1 (z) = A_1 z + b$, where
\[
A_1 = \begin{bmatrix}
a_1^\tp\\
\vdots \\
a_{d}^\tp
\end{bmatrix} \in \mathbb{R}^{d \times p} \quad \text{and} \quad b = \begin{bmatrix}
b_1 \\
\vdots \\
b_{d}
\end{bmatrix} \in \mathbb{R}^{d}.
\]
Let $\alpha_2 (z) = A_2^\tp z + \lambda$, where $A_2^\tp = (c_1, \dots, c_{d}) \in \mathbb{R}^{d}$. Then $x = (x_1, \dots, x_n)^\tp \in \psi_2(\Theta)$ if and only if
\begin{equation}\label{eq:seccor}
\left\{ \begin{array}{ll}
         x_1 & = c_1 \sigma (s_1^\tp a_1 + b_1) + \cdots + c_{d} \sigma (s_1^\tp a_{d} + b_{d}) + \lambda, \\
      \;\vdots &    \quad \qquad \qquad \qquad \qquad \vdots \\
         x_n & = c_1 \sigma (s_n^\tp a_1 + b_1) + \cdots + c_{d} \sigma (s_n^\tp a_{d} + b_{d}) + \lambda .\end{array} \right.
\end{equation}
For $i = 1, \dots, d$, define the vector $y_i = (\sigma (s_1^\tp a_i + b_i), \dots, \sigma (s_n^\tp a_i + b_i))^\tp \in \mathbb{R}^n$, which belongs to $\C_{\max}(S)$ by Corollary~\ref{cor:product}. Thus, \eqref{eq:seccor} is equivalent to
\begin{equation}\label{eq:sec}
x = c_1y_1 + \dots + c_{d}y_{d} + \lambda \mathbbm{1}
\end{equation}
for some $c_1, \dots, c_{d} \in \mathbb{R}$. By definition of secant locus, \eqref{eq:sec} is equivalent to the statement
\[
x \in \Sigma_{d}^{\circ}\bigl(\C_{\max}(S)\bigr) + \lambda \mathbbm{1},
\]
which completes the proof. \qed
\end{proof}

From a practical point of view, we are most interested in basic \emph{topological} issues like whether the image of weights $\psi_k(\Theta)$ is closed or not, since this affects the solvability of \eqref{eq:approx3}. However, the \emph{geometrical} description of $\psi_2(\Theta)$ in  Theorem~\ref{thm:ab1} will allow us to deduce bounds on its dimension. Note that the dimension of the space of weights $\Theta$ as in \eqref{eq:weights} is just $m$ as in \eqref{eq:dim} but this is not the true dimension of the neural network, which should instead be that of the image of weights $\psi_k(\Theta)$.

In general, even for $k =2$, it will be difficult to obtain the exact dimension of $\psi_2(\Theta)$  for an arbitrary two-layer network \eqref{eq:2layernet}. In the next corollary, we deduce from Theorem~\ref{thm:ab1} an  upper bound dependent on  the sample $S \in \mathbb{R}^{n \times p}$ and another independent of it.
\begin{corollary}\label{cor:dimupbound}
For the two-layer network \eqref{eq:2layernet}, we have
\[
\dim \psi_2(\Theta) \le d  (\rk [S, \mathbbm{1}]) + 1,
\]
and in particular,
\[
\dim \psi_2(\Theta) \le \min \big( d (\min(p, n)+1)+1, p n \bigr).
\]
\end{corollary}

When $n$ is sufficiently large and the observations $s_1, \dots, s_n$ are sufficiently general,\footnote{Here and in Lemma~\ref{lem:2ind} and Corollary~\ref{cor:gendim2layer}, `general' is used in the sense of algebraic geometry: A property is general if the set of points that does not have it is contained in a Zariski closed subset that is not the whole space.}  we may deduce a more precise value of $\dim \psi_2(\Theta)$. Before describing our results, we introduce several notations.
For any index set $I \subseteq \{1,\dots,n\}$ and sample $S \in \mathbb{R}^{n \times p}$, we write
\begin{align*}
\mathbb{R}_I^n &\coloneqq \{x \in \mathbb{R}^n :  x_i = 0 \text{ if } i \notin I \text{ and } x_j > 0 \text{ if } j \in I\},\\
F_I(S) &\coloneqq \C_{\max} (S)\cap \mathbb{R}_I^n.
\end{align*}
Note that $F_{\varnothing}(S) = \{0\}$, $\mathbbm{1} \in F_{\{1,\dots,n\}}(S)$, and $\C_{\max}(S)$  may be expressed as
\begin{equation}\label{eq:decomp}
\C_{\max}(S) = F_{I_1}(S) \cup \cdots \cup F_{I_\ell }(S),
\end{equation}
for some index sets $I_1,\dots, I_\ell \subseteq \{1,\dots,n\}$ and $\ell \in \mathbb{N}$ minimum.

\begin{lemma}\label{lem:2ind}
Given a general $x \in \mathbb{R}^n$ and any integer $k$, $1\le k \le n$, there is a $k$-element subset $I \subseteq \{1,\dots,n\}$ and a  $\lambda \in \mathbb{R}$ such that
$\sigma(\lambda \mathbbm{1} + x) \in \mathbb{R}_I^n$.
\end{lemma}
\begin{proof}
Let $x = (x_1, \dots, x_n)^\tp \in \mathbb{R}^n$ be general. Without loss of generality, we may assume its coordinates are in ascending order $x_1 < \cdots < x_n$. For any $k$ with $1\le k \le n$, choose $\lambda$ so that $x_{n - k} < \lambda < x_{n - k + 1}$ where we set $x_0 = - \infty$. Then $\sigma(u - \lambda \mathbbm{1}) \in \mathbb{R}_{\{n-k+1, \dots, n\}}$. \qed
\end{proof}

\begin{lemma}\label{lem:n1ind}
Let $n \ge p + 1$. There is a nonempty open subset of vectors $v_1, \dots, v_{p} \in \mathbb{R}^{n}$ such that for any $p + 1 \le k \le n$, there are a $k$-element subset $I \subseteq \{1,\dots,n\}$ and  $\lambda_1, \dots, \lambda_{p}, \mu \in \mathbb{R}$ where
\[
\sigma(\lambda_1 \mathbbm{1} + v_1), \dots, \sigma(\lambda_{p} \mathbbm{1} + v_{p}), \sigma(\mu \mathbbm{1} + v_1) \in \mathbb{R}_I^n
\]
are linearly independent.
\end{lemma}

\begin{proof}
For each $i =1,\dots,p$, we choose general $v_i = (v_{i, 1}, \dots, v_{i, n})^\tp \in \mathbb{R}^n$ so that
\[
v_{i, 1} < \cdots < v_{i, n}.
\]
For any fixed $k$ with $p + 1 \le k  \le n$, by Lemma~\ref{lem:2ind}, we can find $\lambda_1, \dots, \lambda_{p}, \mu \in \mathbb{R}$ such that $\sigma(v_i - \lambda_i \mathbbm{1}) \in \mathbb{R}_{\{n-k+1, \dots, n\}}$, $i = 1,\dots,n$, and $\sigma(v_1 - \mu \mathbbm{1}) \in \mathbb{R}_{\{n-k+1, \dots, n\}}$. By the generality of $v_i$'s, the vectors $\sigma(v_1 - \lambda_1 \mathbbm{1}), \dots, \sigma(v_{p} - \lambda_{p} \mathbbm{1}), \sigma(v_1 - \mu \mathbbm{1})$ are linearly independent. \qed
\end{proof}

We are now ready to state our main result on the dimension of the image of weights of a two-layer ReLU-activated neural network.
\begin{theorem}[Dimension of two-layer neural network I]\label{thm:opendim}
Let $n \ge d(p + 1) + 1$ where $p$ is the dimension of the input and $d$ is the dimension of the hidden layer. Then there is a nonempty open subset of samples $S \in \mathbb{R}^{n \times p}$ such that the image of weights for the two-layer network \eqref{eq:2layernet} has dimension
\begin{equation}\label{eq:dim2layer}
\dim \psi_2(\Theta) = d (p + 1) + 1.
\end{equation}
\end{theorem}

\begin{proof}
The rows of $S=[s_1^\tp,\dots,s_n^\tp]^\tp \in \mathbb{R}^{n \times p}$ are the $n$ samples $s_1,\dots,s_n \in \mathbb{R}^p$. In this case it will be more convenient to consider the columns of $S$, which we will denote by $v_1, \dots, v_p \in \mathbb{R}^n$. Denote the coordinates by  $v_i = (v_{i, 1}, \dots, v_{i, n})^\tp$, $i =1,\dots,p$. Consider the nonempty open subset
\begin{equation}\label{eq:seqles}
U \coloneqq \{S = [v_1, \dots, v_{p}] \in \mathbb{R}^{n \times p} :  v_{i, 1} < \cdots < v_{i, n},\; i =1,\dots, p\}.
\end{equation}
Define the index sets $J_i \subseteq \{1,\dots,n\}$  by
\[
J_i \coloneqq \{n - i(p + 1) +1, \dots, n\}, \qquad i =1,\dots,d.
\]
By Lemma~\ref{lem:n1ind},
\[
\dim F_{J_i}(S) = \rk [S, \mathbbm{1}] = p + 1, \qquad i=1,\dots,d.
\]
When $S \in U$ is sufficiently general,
\begin{equation}\label{eq:directsum}
\spn  F_{J_1}(S) + \dots + \spn F_{J_{d}}(S) = \spn F_{J_1}(S) \oplus \dots \oplus \spn F_{J_{d}}(S).
\end{equation}
Given any $I \subseteq \{1,\dots,n\}$ with $F_I(S) \ne \varnothing$, we have that for any $x, y \in F_I(S)$ and any $a, b > 0$, $a x + b y \in F_I(S)$. This implies that
\begin{equation}\label{eq:dimspan}
\dim F_I(S) = \dim \spn F_I(S) = \dim \Sigma_r^{\circ} \bigl(F_I(S)\bigr)
\end{equation}
for any  $r \in \mathbb{N}$. Let $I_1,\dots,I_\ell \subseteq \{1,\dots,n\}$ and $\ell \in \mathbb{N}$ be chosen as in \eqref{eq:decomp}. Then
\[
\Sigma_{d}^{\circ} \bigl(\C_{\max}(S)\bigr) = \bigcup_{1 \le i_1 \le \dots \le i_{d} \le \ell } \join \bigl(F_{I_{i_1}}(S), \dots, F_{I_{i_{d}}}(S)\bigr).
\]
Now choose $I_{i_1} = J_1, \dots, I_{i_d} = J_d$. By \eqref{eq:directsum} and \eqref{eq:dimspan},
\[
\dim \join \bigl(F_{J_1}(S), \dots, F_{J_{d}}(S)\bigr) = \sum_{i =1}^d \dim F_{J_i}(S).
\]
Therefore
\begin{align*}
\dim \Sigma_{d}^{\circ} \bigl(\C_{\max}(S)\bigr) &= \dim \join \bigl(F_{J_1}(S), \dots, F_{J_{d}}(S)\bigr) + \dim \spn \{\mathbbm{1}\} \\
&= d \rk [S, \mathbbm{1}] + 1,
\end{align*}
which gives us \eqref{eq:dim2layer}. \qed
\end{proof}

A consequence of Theorem~\ref{thm:opendim} is that the dimension formula \eqref{eq:dim2layer} holds for any general sample $s_1, \dots, s_n \in \mathbb{R}^p$ when $n$ is sufficiently large.

\begin{corollary}[Dimension of two-layer neural network II]\label{cor:gendim2layer}
Let $n \gg pd$. Then for general $S \in \mathbb{R}^{n \times p}$, the image of weights for the two-layer network \eqref{eq:2layernet} has dimension
\[
\dim \psi_2(\Theta) = d (p + 1) + 1.
\]
\end{corollary}
\begin{proof}
Let  the notations be as in the proof of Theorem~\ref{thm:opendim}. When $n$ is sufficiently large, we can find a subset
\[
I = \{i_1, \dots, i_{d(p + 1) + 1}\} \subseteq \{1,\dots,n\}
\]
such that either
\[
v_{j, i_1} < \cdots < v_{j, i_{d(p + 1) + 1}}\qquad \text{or} \qquad v_{j, i_1} > \cdots > v_{j, i_{d(p + 1) + 1}}
\]
for each $j =1,\dots,p$. The  conclusion then follows from Theorem~\ref{thm:opendim}. \qed
\end{proof}

For deeper networks one may have $m\gg \dim \psi_k(\Theta)$ even for $n \gg 0$. Consider a $k$-layer network with one neuron in every layer, i.e.,
\[
d_1 = d_2 = \dots= d_k = d_{k+1}=1.
\]
For any samples $s_1,\dots, s_n\in \mathbb{R}$, we may assume $s_1\leq\dots\leq s_n$ without loss of generality. Then the image of weights $\psi_k(\Theta) \subseteq \mathbb{R}^n$ may be described as follows: a point $x=(x_1,\dots,x_n)^\tp \in \psi_k(\Theta)$ if and only if
\[
x_1=\dots=x_\ell \leq\dots\leq x_{\ell +\ell'}=\dots=x_n
\]
and
\[
x_{\ell +1},\dots,x_{\ell +\ell'-1}\text{ are the affine images of } s_{\ell +1},\dots,s_{\ell +\ell'-1}.
\]
In particular, as soon as $k\geq 3$ the image of weights $\psi_k(\Theta)$ does not change and its dimension remains constant for any $n\geq 6$. 

We next address the case where $q > 1$. One might think that by the $q=1$ case in Theorem~\ref{thm:ab1} and ``one-layer'' case in Corollary~\ref{cor:product},  the image of weights $\psi_2(\Theta)\subseteq \mathbb{R}^{n \times q}$ in this case is simply the direct sum of $q$ copies of  $\join\bigl(\Sigma_{d}^{\circ}(\C_{\max}(S)) , \spn \{ \mathbbm{1} \}\bigr)$. It is in fact only a subset of that, i.e.,
\[
\psi_2(\Theta) \subseteq \bigl\{[x_1,\dots, x_q] \in \mathbb{R}^{n \times q} :
x_1,\dots,x_d \in \join\bigl(\Sigma_{d}^{\circ}(\C_{\max}(S)) , \spn \{ \mathbbm{1} \}\bigr)\bigr\}
\]
but equality does not in general hold.
\begin{theorem}[Geometry of two-layer neural network II]\label{thm:abc}
Consider the two-layer network with $p$-dimensional inputs, $d$ neurons in the hidden layer, and $q$-dimensional outputs:
\begin{equation}\label{eq:3layer}
\mathbb{R}^{p}\xrightarrow{\alpha_1}\mathbb{R}^{d}\xrightarrow{\sigma_{\max}}\mathbb{R}^{d}\xrightarrow{\alpha_2} \mathbb{R}^{q}.
\end{equation}
The image of weights is given by
\begin{multline*}
\psi_2(\Theta) = \bigl\{[x_1,\dots, x_q] \in \mathbb{R}^{n \times q} : \text{there exist } y_1, \dots, y_d \in \C_{\max}(S)  \\
\text{such that } x_i \in \spn \{\mathbbm{1}, y_1, \dots, y_d\}, \; i = 1, \dots, d \bigr\}.
\end{multline*}
\end{theorem}
\begin{proof}
Let $X = [x_1,\dots,x_{q}]\in \psi_2(\Theta)\subseteq \mathbb{R}^{n \times q}$. Suppose that the affine map $\alpha_2:\mathbb{R}^d \to \mathbb{R}^q$ is given by $\alpha_2(x) = Ax + b$ where $A = [a_1,\dots,a_{q}]\in \mathbb{R}^{d \times q}$ and $b = (b_1,\dots,b_q)^\tp\in \mathbb{R}^q$. Then each $x_i$ is realized as in \eqref{eq:sec} in the proof of Theorem~\ref{thm:ab1}. Therefore we conclude that $[x_1,\dots,x_q]\in \psi_2(\Theta)$ if and only if there exist $y_1,\dots, y_{d}\in \C_{\max}(S)$ with
\[
x_i = b_i\mathbbm{1}+\sum_{j=1}^{d} a_{ij} y_j, \qquad i =1,\dots,q,
\]
for some $b_i, a_{ij}\in\mathbb{R}$, $i=1,\dots,n$, $j =1,\dots,d$. \qed
\end{proof}

With Theorem~\ref{thm:abc}, we may deduce analogues of (part of) Theorem~\ref{thm:opendim} and Corollary~\ref{cor:gendim2layer} for the case $q > 1$. The proofs are similar to those of Theorem~\ref{thm:opendim} and Corollary~\ref{cor:gendim2layer}. 
\begin{corollary}[Dimension of two-layer neural network III]\label{cor:dim3}
The image of weights of the two-layer network \eqref{eq:3layer}  with $p$-dimensional inputs, $d$ neurons in the hidden layer, and $q$-dimensional output has dimension
\[
\dim \psi_2(\Theta) =(q+\rk [S, \mathbbm{1}])d+q.
\]
If the sample size $n$ is sufficiently large, then for general $S \in \mathbb{R}^{n \times p}$, the dimension is
\[
\dim \psi_2(\Theta) =(p + q+1)d + q.
\]
\end{corollary}
Note that by \eqref{eq:dim},
\[
\dim \Theta = (p+1)d + (d+1)q = \dim \psi_2(\Theta)
\]
in the latter case of Corollary~\ref{cor:dim3}, as we expect.

\section{Smooth activations}\label{sec:smooth}

For smooth activation like sigmoidal and hyperbolic tangent, we expect the geometry of the image of weights to be considerably more difficult to describe. Nevertheless when it comes to the ill-posedness of the best $k$-layer neural network problem \eqref{eq:approx2}, it is easy to deduce not only that there is a $T \in \mathbb{R}^{n \times q}$ such that \eqref{eq:approx2} does not attain its infimum, but that there there is a  positive-measured set of such $T$'s.

The phenomenon is already visible in the one-dimensional case  $p=q=1$ and can be readily extended to arbitrary $p$ and $q$. Take the sigmoidal activation $\sigma_{\exp}(x) =1/\bigl(1 + \exp(-x)\bigr)$. Let $n=1$. So the sample  $S$ and response matrix $T$ are both in $\mathbb{R}^{1 \times 1} = \mathbb{R}$. Suppose $S \ne 0$. Then for a $\sigma_{\exp}$-activated $k$-layer neural network of arbitrary $k \in \mathbb{N}$,
\[
\psi_k(\Theta)= (0,1).
\]
Therefore any $T\geq 1$ or $T\leq 0$ will not have a  best approximation by points in $\psi_k(\Theta)$.  The same argument works for the hyperbolic tangent activation $\sigma_{\tanh}(x)=\tanh(x)$ or indeed any activation $\sigma$ whose range is a proper open interval. In this sense, the ReLU activation  $\sigma_{\max}$ is special in that its range is not an open interval.

To show that the $n = 1$ assumption above is not the cause of the ill-posedness, we provide a more complicated example with $n = 3$. Again we will keep $p = q =1$ and let
\[
s_1 = 0, \quad s_2 = 1, \quad s_3 = 2; \quad t_1 = 0, \quad t_2 = 2, \quad t_3 = 1.
\]
Consider a $k = 2$ layer neural network with hyperbolic tangent activation
\begin{equation}\label{eq:tanh1}
\mathbb{R} \xrightarrow{\alpha_1} \mathbb{R} \xrightarrow{\sigma_{\tanh}} \mathbb{R} \xrightarrow{\alpha_2}\mathbb{R}.
\end{equation}
Note that its weights take the form
\[
\theta = (a, b, c, d) \in \mathbb{R} \times \mathbb{R} \times \mathbb{R} \times \mathbb{R} \cong \mathbb{R}^{4},
\]
and thus $\Theta = \mathbb{R}^4$. It is also straightforward to see that
\[
\psi_2(\Theta) =  \left\{
\begin{bmatrix*}[r]
c\dfrac{e^b - e^{-b}}{e^b + e^{-b}} + d\\[3ex]
c\dfrac{e^{a+b} - e^{a-b}}{e^{a+b} + e^{a-b}} + d\\[3ex]
c\dfrac{e^{2a+b} - e^{2a-b}}{e^{2a+b} + e^{2a-b}} + d
\end{bmatrix*} \in \mathbb{R}^3 : a,b,c,d \in \mathbb{R} \right\}.
\]

For $\varepsilon > 0$, consider the open set of response matrices
\[
U(\varepsilon) \coloneqq 
\left\{ \begin{bmatrix} t'_1 \\  t'_2 \\ t'_3 \end{bmatrix} \in \mathbb{R}^3 :  |t_1 - t'_1| \le \varepsilon,\; |t_2 - t'_2| \le \varepsilon,\; |t_3 - t'_3| \le \varepsilon \right\}.
\]
We claim that for $\varepsilon$ small enough, any response matrix $T' = (t'_1, t'_2, t'_3)^\tp \in U(\varepsilon)$ will not have a best approximation in $\psi_2(\Theta)$.

Any best approximation of $T = (0, 2, 1)^\tp$ in the closure of $\psi_2(\Theta)$ must take of the form $(0,y,y)^\tp$ for some $y \in [1, 2]$. On the other hand, $(0,y,y)^\tp \notin \psi_2(\Theta)$ for any $y \in [1, 2]$ and thus $T$ does not have a best approximation in $\psi_2(\Theta)$. Similarly, for small $\varepsilon > 0$ and $T' = (t'_1, t'_2, t'_3)^\tp \in U$, a best approximation of $T'$ in the closure of $\psi_2(\Theta)$ must take the form $(t'_1, y, y)^\tp$ for some $y \in [t'_3, t'_2]$. Since $(t'_1, y, y)^\tp \notin \psi_2(\Theta)$ for any $y \in [t'_3, t'_2]$, $T'$ has no best approximation in $\psi_2(\Theta)$. Thus for small enough $\varepsilon > 0$, the infimum  in \eqref{eq:approx2} is unattainable for any $T \in U(\varepsilon)$, a  nonempty open set.

We summarize the conclusion of the above discussion in the following proposition.
\begin{proposition}[Ill-posedness of neural network approximation II]\label{prop:ill2}
There exists  a positive measured set  $U \subseteq \mathbb{R}^{n \times q}$ and some $S \in \mathbb{R}^{n \times p}$ such that the best $k$-layer neural network approximation problem \eqref{eq:approx2} with  hyperbolic tangent activation $\sigma_{\tanh}$  does not attain its infimum for any $T \in U$.
\end{proposition}

We leave open the question as to whether Proposition~\ref{prop:ill2} holds for the ReLU activation $\sigma_{\max}$. Despite our best efforts, we are unable to construct an example nor show that such an example cannot possibly exist.
  
\section{Concluding remarks}

This article studies the best $k$-layer neural network approximation from the perspective of our earlier work \cite{QML}, where we studied similar issues for the best $k$-term approximation.  An important departure from \cite{QML} is that a neural network is not an algebraic object because the most common activation functions $\sigma_{\max}$, $\sigma_{\tanh}$, $\sigma_{\exp}$ are not polynomials; thus the algebraic techniques in \cite{QML} do not apply in our study here and are relevant at best only through analogy.

Nevertheless, by the Stone--Weierstrass theorem continuous functions may be uniformly approximated by polynomials. This suggests that it might perhaps be fruitful to study ``algebraic neural networks,'' i.e., where the activation function $\sigma$ is a polynomial function. This will allow us to apply the full machinery of algebraic geometry to deduce information about the image of weights $\psi_k(\Theta)$ on the one hand and to extend the field of interest from $\mathbb{R}$ to $\mathbb{C}$ on the other. In fact one of the consequences of our results in \cite{QML} is that for an algebraic neural network over $\mathbb{C}$, i.e., $\Theta = \mathbb{C}^m$, any response matrix $T \in \mathbb{C}^{n \times q}$ will almost always have a unique best approximation in $\psi_k(\mathbb{C}^m)$, i.e., the approximation problem \eqref{eq:approx2} attains its infimum with probability one.

Furthermore, from our perspective, the most basic questions about neural network approximations are the ones that we studied in this article but questions like:
\begin{description}
\item[\emph{generic dimension for neural networks}:] for a general  $S \in \mathbb{R}^{n \times p}$ with $n\gg 0$, what is the dimension of $\psi_k(\Theta)$?

\item[\emph{generic rank for neural networks}:] what is the smallest value of $k \in \mathbb{N}$ such that  $\psi_k(\Theta)$ is a dense set in $\mathbb{R}^{n \times q}$?
\end{description}
These are certainly the questions that one would first try to answer about various types of tensor ranks \cite{Landsberg} or tensor networks \cite{YL} but as far as we know, they have never been studied for neural networks. We leave these as directions for potential future work.

\begin{acknowledgements}
The work of YQ and LHL is supported by DARPA D15AP00109 and NSF IIS 1546413. In addition LHL acknowledges support from a DARPA Director's Fellowship and the Eckhardt Faculty Fund. MM would like to thank Guido Montufar for helpful discussions.
\end{acknowledgements}

\bibliographystyle{spmpsci}      

\end{document}